\documentclass[letterpaper, 10 pt, conference]{ieeeconf}  

\IEEEoverridecommandlockouts                              

\usepackage[intlimits]{amsmath}
\usepackage{amsfonts,amssymb}
\DeclareSymbolFontAlphabet{\mathbb}{AMSb}
\usepackage{float}
\usepackage[font=small,labelfont=bf]{caption}
\usepackage{subcaption}
\usepackage{fancyhdr}
\usepackage{fancybox}
\usepackage{ifthen}
\usepackage{lscape,afterpage}
\usepackage{xspace}
\usepackage{epstopdf}
\usepackage{cite}
\usepackage{textcomp}
\usepackage{tabularx}
\usepackage{adjustbox}
\usepackage{array}
\usepackage{makecell}
\usepackage{pgf}
\usepackage{pgffor}
\usepackage{pgfmath}
\usepackage{algorithm}
\usepackage{algpseudocode}
\usepackage{lingmacros}
\usepackage{booktabs}

\usepackage{amsthm}
\usepackage{graphicx}
\usepackage[dvipsnames]{xcolor}
\usepackage{etoolbox}
\newtheorem{theorem}{Theorem}
\newtheorem{proposition}[theorem]{Proposition}
\newtheorem{corollary}{Corollary}[theorem]

\usepackage{xurl}
\makeatletter
\let\NAT@parse\undefined
\makeatother
\usepackage[hidelinks,hypertexnames=false]{hyperref}
\hypersetup{colorlinks=true, linkcolor=blue, citecolor=blue, urlcolor=blue}

\title{\LARGE \bf
Exhaustive-Serve-Longest Control for Multi-robot Scheduling Systems
\thanks{A shorter version of this paper has been submitted to the American Control Conference (ACC 2026) for possible publication.}
}

\author{Mohammad Merati$^{1}$ and David Castañón$^{2}$
\thanks{$^{1}$Divisions of Systems Engineering, Boston University, 8 St Mary's St, Boston, MA 02215, United States. {\tt\small mmerati@bu.edu}}%
\thanks{$^{2}$Department of Electrical and Computer Engineering, Boston University, 8 St Mary's St, Boston, MA 02215, United States. {\tt\small dac@bu.edu}}%
}

\begin{document}
\maketitle
\thispagestyle{empty}
\pagestyle{empty}
\begin{abstract}
We study online task allocation for multi-robot, multi-queue systems with stochastic arrivals and switching delays. Time is slotted; each location can host at most one robot per slot, service consumes one slot, switching between locations incurs a one-slot travel delay, and arrivals are independent Bernoulli processes. We formulate a discounted-cost Markov decision process and propose Exhaustive-Serve-Longest (ESL), a simple real-time policy that serves exhaustively when the current location is nonempty and, when idle, switches to a longest unoccupied nonempty location, and we prove the optimality of such policy. As baselines, we tune a fixed-dwell cyclic policy via a discrete-time delay expression and implement a first-come–first-serve policy. Across server–location ratios and loads, ESL consistently yields lower discounted holding cost and smaller mean queue lengths, with action-time fractions showing more serving and restrained switching. Its simplicity and robustness make ESL a practical default for real-time multi-robot scheduling systems.
\end{abstract}

\section{INTRODUCTION}

Real-time task allocation and routing is a key problem in modern multi-robot systems that arises in logistics, healthcare, maintenance, and urban services. As new tasks arrive, mobile robots repeatedly decide where to go and what to serve, under constraints that require travel/setup times and avoidance of duplication of effort.  Our goal is to develop optimal feedback policies for routing and control of multirobot systems in response to randomly arriving spatially distributed tasks.  

Recent work on hospital logistics considers task assignment and routing for mobile robots incorporating uncertainty in travel times, service times and energy constraints  \cite{cheng2023stochasticAMR,liu2017mobileRobotScheduling,lin2021chargingScheduling}. However, the tasks in these works are known and do not arrive randomly.  The task assignment and routing problem is posed as a dynamic vehicle routing problem (DVRP) with uncertain travel times, and is solved using heuristic techniques. Similar dynamic vehicle routing problems arise in other applications \cite{jia2018setPSO,okulewicz2019continuous,li2010tdvrp, lin2025adaptive,venkatachalam2018uav}, where uncertainty in travel/service times or fuel consumption is incorporated into heuristic search algorithms for routing and task assignment with replanning once information is collected on travel times. 

Online traveling salesperson problems \cite{jaillet2011onlineSF,wen2015OLTSP} address single  robot problems with new task arrivals, but focus primarily on whether arrivals should be accepted or rejected.  If accepted, the replanning algorithm is a fast replanning algorithm with the new task. The work \cite{trigui2014distributed} considers multi-robot task assignment with distributed, market-based algorithms, but new task arrivals require replanning. 

Multi-class queueing networks address the sequencing and routing of tasks in networks as a stochastic control problem.  Although optimal control algorithms are difficult except in simple cases,  \cite{bertsimas1994optimization} develops bounds on achievable performance.  However, servers in these networks are not mobile, and thus fail to capture the inherent tradeoffs in moving servers to address spatially arriving new tasks.

Dynamic vehicle routing problems with stochastic customer requests  are often posed as Markov Decision Processes  \cite{Ulmer2017DeliveryDeadlines}. However, the combinatorial complexity of deterministic vehicle routing problems and the curse of dimensionality leads to the use of approximate dynamic programming (ADP), using either sample scenarios or rollout algorithms.  Furthermore, these ADP algorithms often use routing heuristics to address the combinatorial complexity of vehicle routing problems \cite{Ulmer2017DeliveryDeadlines,Ulmer2019Anticipation,UlmerGMH2019OfflineOnline,Ulmer2020HorizontalADP,HvattumLL2006SSHH}. 

Task allocation with multiple servers and new arrivals is often handled via re-optimization.  Our work in \cite{11077506} solves such problems using dynamic network-flow algorithms, which are fast to allow repeated re-solving as new tasks appear.  Reoptimization approaches work well when newly-arrived tasks are a small fraction of the available tasks, and fail to anticipate the stochastic arrival process.
Distributed auction approaches have been proposed to allocate tasks with deadlines by minimizing total transportation time \cite{Bai2022GroupBasedAuction}, where new task arrivals trigger replanning with reoptimization. 

Approximate dynamic programming using deterministic cost-to-go approximations are used in stochastic multiple knapsack problems \cite{PerryHartman2009DSMKnapsack}.  Another task assignment problem with stochastic arrivals arises in in multichannel allocation for mobile networks \cite{lott2000optimality,Carroll2019ResponseDelays} using index policies for bandit problems.  However, these formulations neglect the switching times associated with assigning incoming tasks to channels.  Other applications with stochastic event arrivals involve persistent monitoring \cite{Yu2015StochasticArrivals}, where stations are visited in a cyclic visit order, and dwell times at stations is controlled based on the observed state of the system.    

Our work in this paper is more closely related to feedback control of multi-class queuing networks.  For such networks, feedback-regulated policies have been developed using fluid approximations and Markov Decision Process techniques in \cite{Meyn2001Sequencing}.  Recent work in \cite{Williams2022MMkSetup} shows deterministic switching times alone can significantly inflate delays in simple M/M/$k$ systems, where routing and assignment is done trivially.  Hofri and Ross \cite{hofri_ross_1987} study the optimal control of a single server, two-queue system with server setup times to switch from serving one queue to another.  


In this work, we study a simple model of online multi-robot task allocation with stochastic per-location arrivals to generalize the results of Hour and Ross \cite{hofri_ross_1987} to multiple robots and multiple arrival locations.  We assume independent task arrivals at each location, modeled in discrete time by Bernoulli processes. We pose the feedback control problem as a discounted, infinite horizon Markov Decision Problem.  Under simple symmetry conditions, we show that an optimal feedback policy, \textbf{ESL (Exhaustive-Serve-Longest)} takes the form of a simple rule:  robots  serve tasks in its current location until exhausted, and then switch to the unattended location with the most remaining tasks.  We prove the optimality of this policy, and benchmark its performance in simulation against other simple feedback rules proposed in the literature:  \emph{FCFS per task} (prioritizes the oldest waiting task across locations) and a \emph{Cyclic} policy with optimized fixed dwell and one-slot travel between partitioned location blocks. 
Our experiments show that, across different numbers of robots and arrival rates, the feedback ESL policy consistently achieves the lowest discounted cost and smallest mean queue lengths, as well as increased fractions of time serving tasks.

\section{Problem Formulation}\label{sec:model}

In this section, we formalize a discrete-time MDP for mobile robots moving among service locations and for the resulting task dynamics at each location\footnote{We Robots/servers and queue/locations interchangeably.}. We consider an $M$–robot, $N$–location model with the restriction that at most one robot may occupy any location in any slot. The objective is to minimize discounted holding cost.

\subsection{System description}
Time is slotted, $t=0,1,2,\dots$. There are $N$ infinite-capacity locations $\{\mathcal{Q}_1,\dots,\mathcal{Q}_N\}$ served by $M$ \emph{non-preemptive} robots, with $M\le N$. At most one task can be completed by a robot in any slot. Let $x_i(t)\in\mathbb{N}_0$ denote the number of waiting tasks at location $i$ at the start of slot $t$, and let $s_r(t)\in\{1,\dots,N\}$ be the location of robot $r$ at the start of slot $t$. The system state can thus be written as $\bigl(s(t);x(t)\bigr)$ with $s(t)=(s_1(t),\dots,s_M(t))$ and $x(t)=(x_1(t),\dots,x_N(t))$.

Tasks arrive independently to locations according to Bernoulli processes. In slot $t$, the indicator
\[
a_i(t)\sim\text{Bernoulli}(p_i),\qquad i=1,\dots,N
\]
takes value $1$ if a task for location $i$ appears, and $0$ otherwise. We adopt the \emph{late-arrival convention}: an arrival in slot $t$ joins $x_i$ at the \emph{end} of that slot and cannot be served before slot $t{+}1$. For bookkeeping, define the cumulative arrivals up to the start of slot $t$ by
\begin{align*}
A_i(t)\;=\;\sum_{k=0}^{t-1} a_i(k),\qquad t\ge 0
\end{align*}

Service is deterministic and consumes one full slot: if a robot begins serving at the start of slot $t$ at its current location, it completes exactly one task there by the end of slot $t$. If a robot switches from its current location to another, it incurs a \emph{deterministic one-slot travel delay} during which no task is served. Also, at most one robot may be assigned to any location in a slot.

We evaluate policies under an infinite-horizon discounted holding cost with discount factor $\beta\in(0,1)$ and unit cost per waiting task per slot. In the symmetric case we take $p_i\equiv p$ for all $i$.

\subsection{State dynamics}

At the \emph{start} of slot $t$ the system state is
\begin{equation*}\label{eq:MsrvNq-state}
  z(t)
  =\Bigl(\,s_1(t),\dots,s_M(t)\,;\;x_1(t),\dots,x_N(t)\Bigr)
\end{equation*}
where $s_r(t)\in\{1,\dots,N\}$ is the location of robot $r=1,\dots,M$, and $x_i(t)\in\{0,1,2,\dots\}$ is the number of jobs in location $i=1,\dots,N$. The \emph{no co–location} constraint is
\begin{equation}\label{eq:no-colocation}
  \Bigl|\{\,r: s_r(t)=i\,\}\Bigr|\ \le\ 1
  \qquad\text{for each }i=1,\dots,N
\end{equation}
At the start of slot $t$ a control
\(
  u(t)=\bigl(u_1(t),\dots,u_M(t)\bigr)
\)
is chosen, where each robots’s action is
\begin{align*}
u_r(t)\in
\begin{cases}
\{\textbf{serve},\textbf{idle}, \textbf{switch(j)}: j\neq s_r(t)\} & x_{s_r(t)}(t)>0\\
\{\textbf{idle},\textbf{switch(j)}: j\neq s_r(t)\} & x_{s_r(t)}(t)=0
\end{cases}
\end{align*}

The joint action must satisfy two feasibility conditions:
(i) \emph{Per–location service feasibility} (one job per slot at any visited location):
\begin{align*}
  \sum_{r=1}^M \mathbf 1\{\,s_r(t)=i,\ u_r(t)=\textbf{serve}\,\}
  \ \le\ 1,
  \qquad i=1,\dots,N
\end{align*}

(ii) \emph{No collision at next slot} (at most one robot occupies any location at $t\!+\!1$):
\begin{align}\label{eq:no-collision}
  &\sum_{r=1}^M \mathbf 1 \bigl\{ s_r(t)=i,\ u_r(t)\in\{\textbf{serve},\textbf{idle}\}\bigr\} \notag\\
  &+\sum_{r=1}^M \mathbf 1\bigl\{u_r(t)=\textbf{switch}(i)\bigr\}\le\ 1, \quad i=1,\dots,N
\end{align}

If $s_r(t)=i$ and $u_r(t)=\textbf{serve}$ then one job is completed from location $i$ during slot $t$;
if $u_r(t)=\textbf{switch}(j)$ then robot $r$ travels the entire slot; if $u_r(t)=\textbf{idle}$ no service occurs.  Arrivals $a_i(t)\in\{0,1\}$ are Bernoulli$(p_i)$ and occur at the end of slot $t$.
Define the location–level departure indicator
\[
  d_i(t)
  := \sum_{r=1}^M \mathbf 1\{\,s_r(t)=i,\ u_r(t)=\textbf{serve}\,\}
  \ \in\ \{0,1\}
\]
And the cumulative departures up to the start of slot $t$ is:
\begin{align*}
D_i(t)\;=\;\sum_{k=0}^{t-1} d_i(k),\qquad t\ge 0
\end{align*}
With these conventions the dynamics are
\begin{equation}\label{eq:MsrvNq-update-nocoloc}
\begin{aligned}
  x_i(t+1) &= x_i(t) - d_i(t) + a_i(t), && i=1,\dots,N\\[4pt]
  s_r(t+1) &=
  \begin{cases}
     s_r(t) & u_r(t)\in\{\textbf{serve},\textbf{idle}\}\\[2pt]
     j      & u_r(t)=\textbf{switch}(j)
  \end{cases}
  && r=1,\dots,M
\end{aligned}
\end{equation}
Given an arrival sample path $\{A_i(t)\}$, the recursion \eqref{eq:MsrvNq-update-nocoloc} together with \eqref{eq:no-colocation}–\eqref{eq:no-collision} fully specifies the evolution under any admissible control sequence.
\subsection{Performance criterion}

Let
\[
  \mathcal S=\{0,1,\dots,N\}^{M}\times \{0,1,2,\dots\}^{\,N}
\]
be the state space, and $\mathcal U(z)$ the set of joint controls that
satisfy \eqref{eq:no-colocation}–\eqref{eq:no-collision} at state $z$.
For a stationary deterministic policy
$\pi:\mathcal S\to\mathcal U(\mathcal S)$ define the one–period cost
\[
  c\bigl(z(t)\bigr) \;=\; \sum_{i=1}^{N} x_i(t)
\]
and the $\beta$–discounted criterion
\[
  V_\pi(z)
  \;=\;
  \mathbb E_z^{\pi}\!\left[\,
      \sum_{t=0}^{\infty} \beta^{\,t}\, c\bigl(z(t)\bigr)
  \right],
  \qquad
  V^*(z)=\inf_{\pi} V_\pi(z)
\]
Because the state space is countable and action set is finite, existence of an optimal stationary deterministic policy is guaranteed \cite{Puterman2014}, and we propose a policy that provably optimizes the given discounted cost.

\section{Exhaustive Service and Switch-on-Empty}
This section establishes two key structural properties of optimal control in our model. First, when a robot is at a location with pending tasks, serving is strictly better than idling or switching. Second, when a robot is at an empty location and another location has tasks, switching is strictly better than idling.

\begin{proposition}\label{prop:MsrvNq-exhaust}
Consider the $M$–robots, $N$–locations model. Let $g$ be any stationary deterministic policy under which, at some decision epoch and state
\[
  z(0)=(s_1,\dots,s_M;\,x_1,\dots,x_N)
\]
a particular robot $r$ is located at location $i$ with $x_i>0$ and $g$ either \textbf{switches} $r$ away from $i$
or \textbf{idles} $r$. Then $g$ is strictly suboptimal.
\end{proposition}

\begin{proof}
Without loos of generality, let's assume that we start with the initial state 
$z(0)=\bigl(1,s_2(0),\dots,s_M(0);m_1,m_2,\dots,m_N\bigr)$.
We define two systems $G$ and $\Pi$, and couple them on the same arrival sample path, with the same initial state $z(0)$, acting under policies $g$ and $\pi$, respectively. Let $D^\rho(t)$ denote cumulative departures under $\rho\in\{g,\pi\}$, and let $x^\rho(t)=\sum_{k=1}^N x_k^\rho(t)$ be the total backlog at the start of slot $t$. Because $z(0)$ and arrivals are coupled, for $t\ge0$:
\begin{align}\label{eq:deps_backlogs}
  x^g(t)-&x^{\pi}(t) =\bigl[x^g(0)-x^\pi(0)\bigr] + [A^g(t)-A^\pi(t)] \notag \\
  &- \bigl[D^g(t)-D^\pi(t)\bigr] = D^\pi(t)-D^g(t)
\end{align}

\noindent\textbf{(A) Switching away while $x_1>0$.}
Define policy $g$ as follows: for each robot $s$ at location $\ell$, $g$ serves for some time and then switches to the next available destination $j_s^\star(z)$ (e.g., the first unoccupied location in cyclic order available at $t{+}1$). Without loss of generality, assume robot~1 is at location~1 with $x_1>0$ and that under $g$ it switches immediately at $t=0$. 

Define policy $\pi$ to agree with $g$ for all robots and all times except for robot~1 at location~1: under $\pi$, robot~1 first serves one task at $t=0$, and thereafter follows exactly the same sequence of actions as robot~1 under $g$, but delayed by one slot. Whenever any other robot arrives to location~1, choose thresholds so that under $g$ it serves until $m_1$ and under $\pi$ it serves until $m_1{-}1$, then switches. This synchronization ensures that every such “other” robot leaves location~1 at the same time under both policies, and that (up to the coupling time defined below) the queue length at location~1 under $g$ is always exactly one larger than under $\pi$.

With these definitions, the cumulative–departure gap satisfies $D^\pi(t)-D^g(t)\in\{0,1\}$, equal to $1$ whenever $g$ has already switched robot~1 away from a location but $\pi$ has not yet done so, and equal to $0$ when $\pi$ is performing the corresponding one–slot–delayed switch. Let $\tau$ be the first time at which robot~1 under $g$ returns to location~1. By construction, robot~1 under $\pi$ returns exactly one slot later, at $\tau{+}1$. At time $\tau$, robot~1 under $g$ is available at location~1; at time $\tau{+}1$ it serves one task, and at that same time robot~1 under $\pi$ arrives, eliminating the one–job advantage. From $\tau{+}2$ onward, the two systems are coupled perfectly (identical states and actions), and we set $\pi$ to mimic $g$ thereafter.

Consequently,
\[
D^\pi(t)-D^g(t)=
\begin{cases}
0,& t=0,\\[3pt]
1,& t=1,\\[3pt]
\ge 0,& 2\le t\le \tau,\\[3pt]
0,& t\ge \tau{+}1.
\end{cases}
\]
Since this holds for every sample path and using \eqref{eq:deps_backlogs}, discounted summation and expectation yield
\begin{align*}
V_g(z_0)-V_\pi(z_0) &= \mathbb{E}\left[\sum_{t=0}^\infty \beta^{t}\bigl(x^g(t)-x^\pi(t)\bigr)\right] \\
&=\mathbb{E}\left[\sum_{t=0}^\infty \beta^{t}\bigl(D^\pi(t)-D^g(t)\bigr)\right] \ge\ \beta>0
\end{align*}
showing that switching away from available work is strictly dominated by serving first.

\medskip
\noindent\textbf{(B) Idling while $x_1>0$.}
Redefine the two policies as follows. Under $g$, robot~1 idles at time $t=0$ while located at~1. Under $\pi$, robot~1 instead serves one job at $t=0$. From $t\ge 1$ onward, all robots (including robot~1) follow exactly the same actions under $g$ and $\pi$, so robot locations coincide at every time, and queue lengths at all locations except~1 are identical under the two policies.

As in part~(A), let $\tau$ be the first time at which robot~1 (under both policies) returns to location~1. At time $\tau{+}1$, let $g$ serve one job at~1 while $\pi$ idles; this eliminates the one–job advantage accumulated by $\pi$ up to that point. From $t\ge \tau{+}2$ we couple the systems perfectly (identical states and actions thereafter). Since the two systems share the same initial state and arrival sample path, the cumulative–departure difference obeys
\[
D^\pi(t)-D^g(t)\;=\;
\begin{cases}
0, & t=0,\\[3pt]
1, & 1\le t \le \tau,\\[3pt]
0, & t\ge \tau{+}1.
\end{cases}
\]
Therefore, for every sample path,
\begin{align*}
V_g(z_0)-V_\pi(z_0) &= \mathbb{E}\!\left[\sum_{t=0}^\infty \beta^{t}\bigl(x^g(t)-x^\pi(t)\bigr)\right] > 0
\end{align*}
Which is in contradiction with the assumption, making the idle policy sub-optimal when the working location is not empty.

\end{proof}

\begin{proposition}\label{prop:MsrvNq-noIdle}
Consider an $M$–robot, $N$–location system with identical Bernoulli$(p)$ arrivals for all locations. Fix a decision epoch and suppose a particular robot $r$ is at location $i$ with $x_i=0$. If there exists a location with $x_j>0$ that is unoccupied, then idling $r$ is strictly suboptimal.
\end{proposition}

\begin{proof}
Without loss of generality, fix time $t=0$ and robot $r=1$ with initial state
\[
  z(0)\;=\;\bigl(1,\,s_2(0),\dots,s_M(0);\ 0,\,m_2,\dots,m_N\bigr)
\]
so location~1 is empty for robot~1 at $t=0$. Let $j^\star=2$ denote the next higher–index location that is unoccupied at $t=0$. We define two policies directly on this initial condition with a one–slot time shift. For policy $g$, at $t=0$, robot~1 idles at location~1. For all robots and all subsequent slots $t\ge 1$, apply the rule “\textbf{serve} if the current location is nonempty; otherwise \textbf{idle}." For policy $\pi$, identical to $g$ for all robots except robot~1 at $t=0$, which instead executes $\textbf{switch}(j^\star)$. For $t\ge 1$, $\pi$ follows the same “\textbf{serve} if nonempty, else \textbf{idle}” rule as $g$ for all robots (including robot~1).


Let $\Omega=(\{0,1\}^N)^{\mathbb N}$ be the arrival-sequence space with the product measure~$P$. Define the mirror map
\[
\begin{aligned}
\phi : \Omega &\;\longrightarrow\; \Omega, \\[4pt]
\omega
      &= \bigl((a_1(t),a_2(t),a_3(t),\dots,a_N(t))\bigr)_{t\ge 0}
        \;\longmapsto \\[4pt]
\phi(\omega)
      &= \bigl((a_2(t),a_1(t),a_3(t),\dots,a_N(t))\bigr)_{t\ge 0}
\end{aligned}
\]
Because the two arrival streams are i.i.d.\ Bernoulli($p$), swapping the first two arrival processes does not change the overall joint probability; hence $P\circ\phi^{-1}=P$, which means that the mirroring transformation is measure-preserving.
If we assign the arrival realization $\omega$ to system $G$, and $\phi(\omega)$ to system $\Pi$, the difference between costs for these systems looks like the following:
\begin{equation}
\resizebox{\linewidth}{!}{$
\begin{aligned}
&V_g(1,\dots,s_M;0,m_2,\dots,m_N) - V_\pi(1,\dots,s_M;0,m_2,\dots,m_N) \notag \\
&=\beta\Bigl[V_g\bigl(1,\dots,s_M;a_1(0),m_2+a_2(0),\dots,m_N + a_N(0)\bigr) \notag \\
&-V_\pi\bigl(2,\dots, s_M;a_2(0),m_2+a_1(0),\dots,m_N + a_N(0)\bigr)\Bigl] \notag
\end{aligned}
$}
\end{equation}

Both $g$ and $\pi$ are exhaustive and they serve the same number of tasks as long as arrival occurs at their working locations. And at some time $t$, when all the arrivals at the working locations have been cleared except $a_1(t-1)$, the difference looks like following:

\begin{equation}\label{eq:diff}
\resizebox{\linewidth}{!}{$
\begin{aligned}
&V_g(1,\dots,s_M;0,m_2,\dots,m_N) - V_\pi(1,\dots,s_M;0,m_2,\dots,m_N) \\
&= \beta^t\Bigl[V_g\bigl(1,\dots,s_M;a_1(t-1),m_2+A_2(t),\dots,m_N + A_N(t)\bigr) \\
&\quad- V_\pi\bigl(2,\dots,s_M;A_2(t),m_2+a_1(t-1),\dots,m_N + A_N(t)\bigr)\Bigr]
\end{aligned}
$}
\end{equation}

For the case of $m_2=0$, the two systems coincide and the value difference is $0$, meaning that there's no difference between switching or idling at that initial state. \\
But for the case of $m_2 > 0$, once there is no arrival at the working location, at least one extra task will be served under policy $\pi$ while $g$ idles the robot, and this makes the \eqref{eq:diff} positive from that time beyond. This contradicts our initial assumption, making switching an optimal policy when the working queue is empty.
\end{proof}
So far, we proved serving tasks in a location until exhaustion and then switching is optimal. Now, we want to prove a proposition which shows that switching to shorter queues are sub-optimal. To show the gist of the argument we put the proof for the case of one robot here, and for the multi-robot version, we recommend the reader to look at appendix.

\begin{proposition}\label{prop:longer_is_better_action}
Consider a one-robot, $N$–queue system with identical Bernoulli(p) arrivals for all locations. Upon exhausting the currently served location $k$ (i.e., $x_k=0$), if there exist two other non-empty locations $i$ and $j$ with initial lengths $x_i < x_j$, it is strictly suboptimal to switch to the shorter location, $i$.
\end{proposition}

\begin{proof}
We prove by contradiction and using a sample-path coupling argument. Without loss of generality, let's assume that $i=1$ and $j=2$. Assume that it is optimal to switch to the shorter queue, $1$. We will construct two policies, $g$ and $\pi$.
Let the system start in state $z(0) = (k, x_1, x_2, \dots, x_k=0, x_N)$ where $x_2 > x_1 > 0$.

\smallskip
\noindent\textbf{System $G$ (policy $g$):}
\begin{itemize}
    \item At $t=0$, policy $g$ initiates a $\textbf{switch}(1)$.
    \item From $t=1$ onward, it serves queue $1$ exhaustively. Let $\tau$ be the first time slot at which queue $1$ is empty again under this policy.
    \item After exhausting queue $1$, it proceeds to switch the server arbitrary and serve other queues exhaustively.
\end{itemize}

\smallskip
\noindent\textbf{System $\Pi$ (under policy $\pi$):}
\begin{itemize}
    \item At $t=0$, policy $\pi$ initiates a $\textbf{switch}(2)$.
    \item To facilitate a direct comparison, we couple System $\Pi$ to an altered arrival process. Let $\omega$ be the arrival realization for System G. System $\Pi$ is subject to the arrival realization $\phi_{12}(\omega)$, where the arrival streams for queues $1$ and $2$ are swapped at all time steps. Since the arrival processes are i.i.d. and symmetric, this transformation is measure-preserving and does not change the expected costs.
    \item From $t=1$ onward, policy $\pi$ serves queue $2$, until time $\tau$. After this time, policy $\pi$ serves only $x_2-x_1$ tasks from queue $2$, ignoring any new task that arrives after.
    \item After completing the $x_2-x_1$ services in queue $2$, it begins to mimic policy $g$. It switches to the same sequence of queues as $g$, and serves for the exact same number of time slots at each queue as $g$ does, but with the roles of queues $1$ and $2$ swapped.
    \item At the very first time that policy $\pi$ switches to queue $1$, which means policy $g$ has already switched to queue $2$, policy $\pi$ will do the same actions as $g$ beyond.
\end{itemize}
With this definition, policy $\pi$ will always be behind $x_2-x_1$ periods in terms of the switching time. Also, at the time that servers under both policy $\pi$ and $g$ have left a queue, there will the same number of tasks remaining in that queue. And, at the first time that policy $g$ switching to queue $j$ and policy $\pi$ switches to $i$, after that $x_j-x_i$ periods, both systems will be at the same state beyond. These allow for a direct comparison.

Using \eqref{eq:deps_backlogs}, we know that 
\begin{align*}
V_g(z_0) - V_\pi(z_0) &= \mathbb E\left[\sum_{t=0}^{\infty} \beta^t \left(x^g(t) - x^\pi(t)\right)\right] \\
&= E\left[\sum_{t=0}^{\infty} \beta^t \left(D^\pi(t) - D^g(t)\right)\right]
\end{align*}
Until time $\tau$ both systems have the same cost, and hence:
\begin{align*}
&V_g(k,x_1,x_2,\dots,x_k=0,\dots,x_N) \\
&- V_\pi(k,x_1,x_2,\dots,x_k=0,\dots,x_N) \\
& = \beta^\tau \bigl[V_g(1,0,x_2+A_2(\tau),\dots,x_N) \\
& - V_\pi(2,x_1+A_2(\tau),x_2+A_1(\tau)-\tau,\dots,x_N)\bigr]
\end{align*}
Since $x_1+A_1(\tau) - \tau = 0$ we can say $A_1(\tau)-\tau = -x_1$.
\begin{align*}
&V_g(k,x_1,x_2,\dots,x_N) - V_\pi(k,x_1,x_2,\dots,x_N) \\
& = \beta^\tau \bigl[V_g(1,0,x_2+A_2(\tau),\dots,x_N) \\
& - V_\pi(2,x_1+A_2(\tau),x_2-x_1,\dots,x_N)\bigr]
\end{align*}
We use change of variable $x_2-x_1=k'>0$, which gives
\begin{align*}
&V_g(k,x_1,x_2,\dots,x_N) - V_\pi(k,x_1,x_2,\dots,x_N) \\
& = \beta^\tau \bigl[V_g(1,0,x_2+A_2(\tau),\dots,x_N) \\
& - V_\pi(2,x_2-k'+A_2(\tau),k',\dots,x_N)\bigr]
\end{align*}
Since policy $g$ does a switch at the next step, it will have at least one less departure until some time $\tau'$ that policy $\pi$ switches after completing $k'$ services. This will make the difference positive for a while:
\begin{equation}
\resizebox{\linewidth}{!}{$
\begin{aligned}
&V_g(k,x_1,x_2,\dots,x_N) - V_\pi(k,x_1,x_2,\dots,x_N) \notag \\
& = \mathbb E[\sum_{t=\tau+1}^{\tau'}\beta^t] + \beta^{\tau'+1}\mathbb E\Bigl[V_g\bigl(z_1(\tau'+1)\bigr)-V_g\bigl(z_2(\tau'+1)\bigr)\Bigr] \notag \\
& = \mathbb E[\sum_{t=\tau}^{\tau'}\beta^t] + \beta^{\tau'+1} \mathbb E\Bigl[\sum_{t=\tau'+1}^{\infty}\bigl(D^\pi(t)-D^g(t)\bigr)\Bigr]
\end{aligned}
$}
\end{equation}

After the time $\tau'$, since policy $\pi$ will serve at each queue for a deterministic amount of time, we are guaranteed to keep the $D^\pi(t)-D^g(t)$ non-negative, since policy $\pi$ will always does its next switch exactly $k'$ periods later than $g$. Therefore:
\[
V_g(k,x_1,x_2,\dots,x_N) > V_\pi(k,x_1,x_2,\dots,x_N)
\]
This shows that switching to the shorter queue is sub-optimal.
\end{proof}

\begin{corollary}\label{cor:start_at_longer_robot}
Fix a decision epoch at which a robot $r$ can begin service immediately at one of two \emph{unoccupied} nonempty locations $i$ and $j$, with $x_i<x_j$. It is strictly better to begin service at the longer location $j$.
\end{corollary}

\begin{proof}
Apply Proposition~\ref{prop:longer_is_better_action} with the common initial travel slot removed (both alternatives require zero travel here). Equivalently, repeat the same sample-path swap argument starting at the first serving slot: since $x_j>x_i$, serving $j$ yields (weakly) more departures than serving $i$ over the interval before the first switch, with a strict gain on a set of paths of positive probability. The discounted backlog gap is therefore strictly positive.
\end{proof}

\begin{corollary}\label{cor:assign_longest}
At any decision epoch with a set $\mathcal{R}$ of $K\!\ge\!1$ robots to assign and a set $\mathcal{S}$ of \emph{unoccupied} nonempty locations, any collision-free assignment that does not cover the $K$ longest locations in $\mathcal{S}$ is strictly dominated by one that does (ties broken arbitrarily). In particular, if $K\le |\mathcal{S}|$, an optimal instantaneous placement assigns the $K$ robots to $K$ distinct longest locations.
\end{corollary}

\begin{proof}
Suppose some feasible assignment leaves a longer location $j$ uncovered while assigning a robot $r$ to a shorter location $i$ with $x_i<x_j$. By Proposition~\ref{prop:longer_is_better_action}, redirecting $r$ from $i$ to $j$ strictly improves the discounted cost (other robots unchanged). Repeating this pairwise exchange eliminates all such inversions; the process terminates with robots covering the $K$ longest locations. Each exchange yields a strict improvement on a set of sample paths of positive probability, so the final assignment strictly dominates the initial one.
\end{proof}

\begin{theorem}\label{thm:ESL-opt}
Consider the $M$–robots, $N$–locations model with i.i.d.\ Bernoulli$(p)$ arrivals at each location,
unit service time, one–slot switching delay, no co–location, unit holding cost, and discount factor $\beta\in(0,1)$. There is an optimal Exhaustive–Serve–Longest policy as following:
\begin{enumerate}
\item If a robot is at a nonempty location, it \textbf{serves} exhaustively.
\item If a robot is at an empty location, it \textbf{switches} to a longest unoccupied nonempty location.
\end{enumerate}
\end{theorem}

\begin{proof}
By Proposition~\ref{prop:MsrvNq-exhaust}, any policy that \emph{switches} away or \emph{idles} at a location
with $x_i>0$ is strictly suboptimal. Hence, at any nonempty current location, \textbf{serve} is uniquely optimal. When the current location is empty, Proposition~\ref{prop:longer_is_better_action} shows that \textbf{idling} is strictly suboptimal whenever a feasible switch to a nonempty location exists. Thus an optimal action must \textbf{switch} to some nonempty location. Given two candidate nonempty destinations $i$ and $j$ with $x_i<x_j$, Proposition~\ref{prop:longer_is_better_action}
establishes that switching to the shorter one is strictly suboptimal. If the robot can start serving immediately at either of two nonempty locations, Corollary~\ref{cor:start_at_longer_robot} says to start at the longer one. With multiple robots, Corollary~\ref{cor:assign_longest} yields a pairwise–exchange improvement:
any collision–free assignment that fails to cover the $K$ longest available locations is strictly dominated by one that does. Hence the optimal joint action assigns robots to distinct longest nonempty locations.

Combining these parts gives the ESL rule: serve when nonempty; otherwise switch to a longest unoccupied nonempty location.
\end{proof}

\section{Experiments and results}
We empirically evaluate three robot-routing policies on a stylized queueing environment and compare their performance across traffic loads and robot-to-location ratios. We first describe the simulation setting (\S\ref{sec:sim-scenario-queues}), then the evaluation metrics (\S\ref{sec:metrics-queues}), and finally the results (\S\ref{sec:results-queues}).

\subsection{Simulation Scenarios}
\label{sec:sim-scenario-queues}
The environment consists of $N$ independent and identical queues (locations) and $M$ identical robots (servers). In every slot, a robot located at queue $i$ may (i) \textbf{serve} one task if available, (ii) \textbf{switch} to another location, or (iii) \textbf{idle} at its current location. Switching incurs a deterministic travel delay of one slot during which no service is provided.

We study policies that make decisions at each slot under identical information:
\begin{itemize}
    \item \textbf{Exhaust and Switch to Longest (ESL)}: each robot first serves if its current location is nonempty; otherwise robots are assigned to the longest distinct locations (ties broken by lower index).
    \item \textbf{First-Come-First-Serve per Task (FCFS)}: each location maintains per-task ages; at each slot, robots prioritize the location whose \emph{oldest waiting task} has the largest age. Ties favor staying at one’s current location, then lower index.
   \item \textbf{Cyclic with Fixed Dwell (cyclic)}: locations are partitioned evenly among robots; each robot cycles through its assigned subset, spending a fixed dwell time at each location before switching \cite{Yu2015StochasticArrivals}. The dwell time can be set by solving an optimization problem with two objectives: (i) balancing the average number of events to be collected at each station so that no station receives insufficient or excessive monitoring effort, and (ii) minimizing the maximum delay in observing two consecutive events generated by the same process between policy cycles. Since our problem is symmetric in terms of the arrival rates and travel time between locations, we can allocate $n = N/M$ locations to each robot to monitor, set dwell times as $t_1=\dots=t_N=t= T_{obs}/n$, and solve the 
   \[
    \underset{T_{\mathrm{obs}}>0}{\operatorname*{arg\,min}}\;
    \frac{T_{\mathrm{obs}}+n-\dfrac{T_{\mathrm{obs}}}{n}+\dfrac{T_{\mathrm{obs}}}{n}\,\bigl(1-p\bigr)^{T_{\mathrm{obs}}/n}}
    {1-\bigl(1-p\bigr)^{T_{\mathrm{obs}}/n}}
    \]
   to find the total monitoring time for each robot.
\end{itemize}

We vary the robot–location ratio $\rho = M/N \in \{1/3, 1/2\}$ and the traffic intensity via a common per-queue arrival probability $p \in \{\alpha\rho: \alpha \in \{0.2, 0.5, 0.8\}\}$ to span light-to-heavy loads while maintaining $p<1$. We use identical $p_i\!=\!p$ across queues. Each experiment runs for $10000$ slots per episode and is replicated for $100$ i.i.d.\ episodes.

\subsection{Evaluation Metrics}
\label{sec:metrics-queues}
We report:
\begin{itemize}
    \item \textbf{Discounted expected holding cost} (lower is better): per-episode discounted sum of total queue length, averaged over episodes; reported as mean $\pm$ 95\% CI.
    \item \textbf{Mean queue length per location} (undiscounted): within-episode time average of $\tfrac{1}{N}\sum_{i=1}^{N} x_i(t)$, then averaged across episodes; reported as mean $\pm$ 95\% CI.
    \item \textbf{Action-time fractions}: fraction of robot-time spent \emph{serving}, \emph{switching}, and \emph{idling}, computed per episode over $M\cdot T$ decisions, then averaged across episodes; each reported as mean $\pm$ 95\% CI.
\end{itemize}

\subsection{Results}
\label{sec:results-queues}
We evaluate the three policies on the grid $\rho\in\{1/3,1/2\}$ and $\alpha\in\{0.2,0.5,0.8\}$ (with $p=\alpha\rho$). For each $(\rho,\alpha)$, we show three grouped bar charts (one per metric) with error bars indicating 95\% CIs. In the following we summarize qualitative trends; the full numerical results (means and CI) are available in tables \ref{tab:m6r2} and \ref{tab:m6r3}.

Across all scenarios, the Exhaustive-Serve-Longest (ESL) policy consistently dominates the alternatives. With $N{=}6$ and $M{=}2$ (top row of Fig.~\ref{fig:summary-grid}), ESL achieves the lowest discounted holding cost and the smallest mean queue length under \emph{light} and \emph{moderate} loads, while also exhibiting the most favorable action mix (highest serving fraction, lower switching). FCFS ranks second in these regimes and outperforms the cyclic strategy. Under \emph{high} load, however, FCFS degrades, its serving fraction drops while switching rises, leading to larger backlogs and higher cost; in this regime the cyclic policy overtakes FCFS, as its fixed dwell yields higher serving and lower switching than FCFS. ESL remains the best throughout.

With $N{=}6$ and $M{=}3$, the ordering is stable across all loads: ESL $>$ FCFS $>$ cyclic on both discounted cost and mean queue length. The action fractions explain this: ESL consistently delivers the largest serving share, followed by FCFS, with cyclic serving least and switching most. Comparing $N{=}3$ vs.\ $N{=}2$ at high load highlights the same mechanism behind the FCFS--cyclic crossover: adding a robot raises the serving fraction under FCFS sufficiently to surpass cyclic; with fewer robots ($N{=}2$), FCFS spends comparatively more time switching, allowing cyclic to edge it out at high load. Under heavy traffic, \textbf{ESL} dominates by a wide margin: its discounted cost and mean queue length are an order of magnitude lower than \textsc{FCFS} and \textsc{Cyclic} (e.g., \(795.85\) vs.\ \(3890.03/4924.27\); \(1.45\) vs.\ \(289.19/334.71\)). This gap aligns with action fractions: ESL sustains a higher serving share with restrained switching (\(0.7995\) serve, \(0.1874\) switch), whereas FCFS/Cyclic spend more time switching, allowing backlogs to explode at high load. Overall, ESL’s ability to keep robots on the most productive queues (serve-in-place when nonempty, coordinated assignment otherwise) drives its uniform advantage across ratios and loads.

\begin{figure*}[t]
\centering

\begin{subfigure}{0.32\textwidth}
  \centering
  \includegraphics[width=\linewidth]{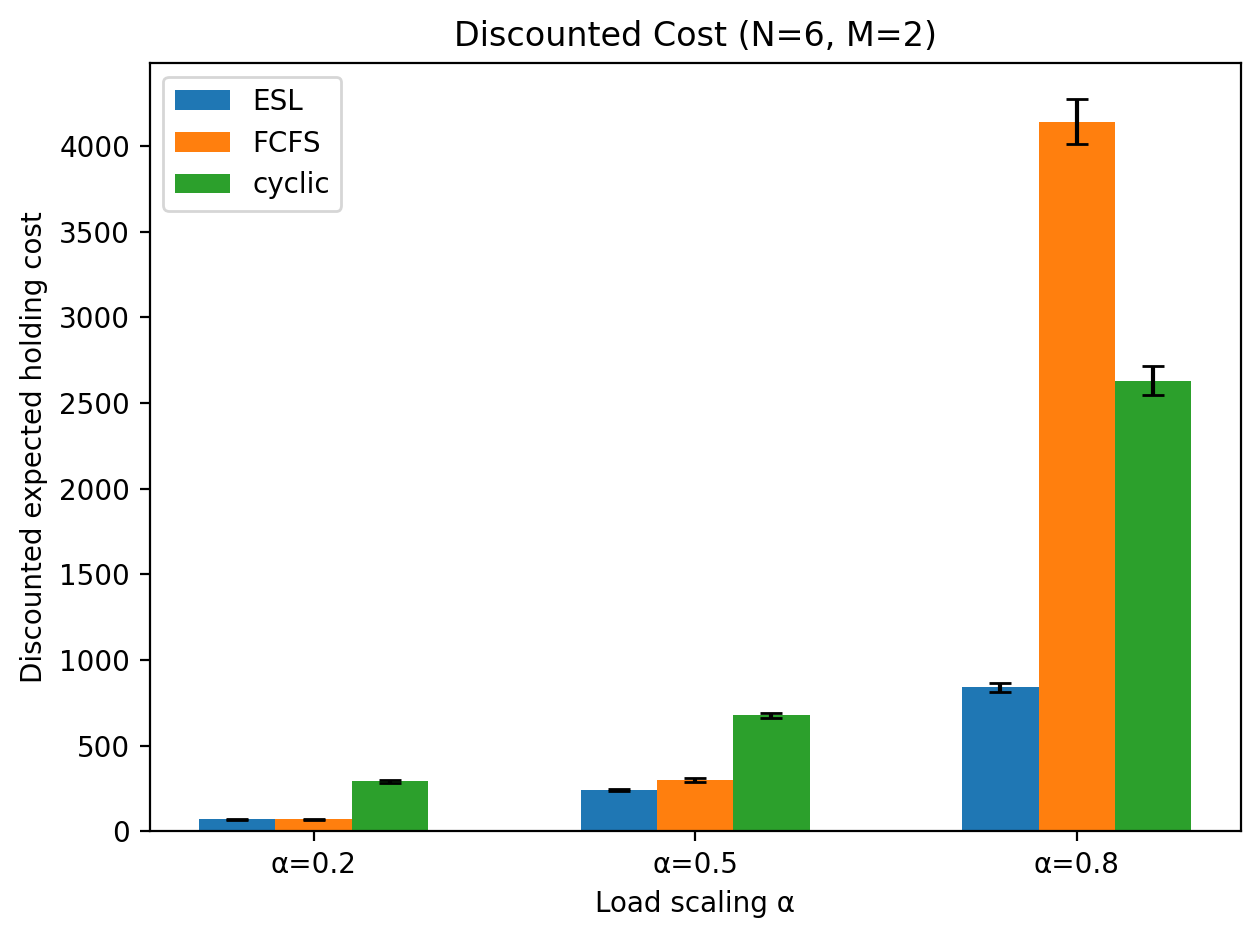}
  \caption{Discounted cost, $N=6,\ M=2$}
\end{subfigure}\hfill
\begin{subfigure}{0.32\textwidth}
  \centering
  \includegraphics[width=\linewidth]{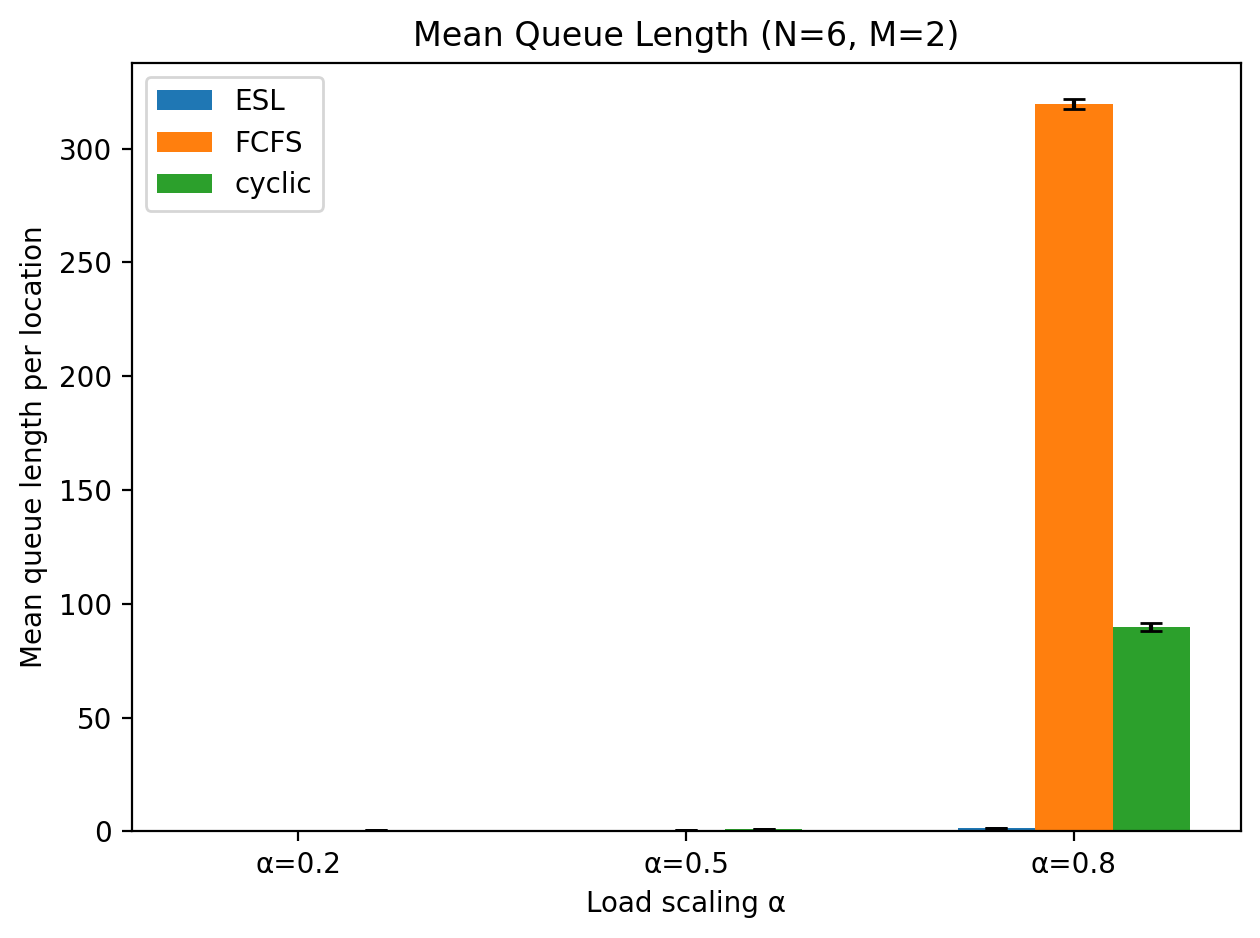}
  \caption{Mean queue length, $N=6,\ M=2$}
\end{subfigure}\hfill
\begin{subfigure}{0.32\textwidth}
  \centering
  \includegraphics[width=\linewidth]{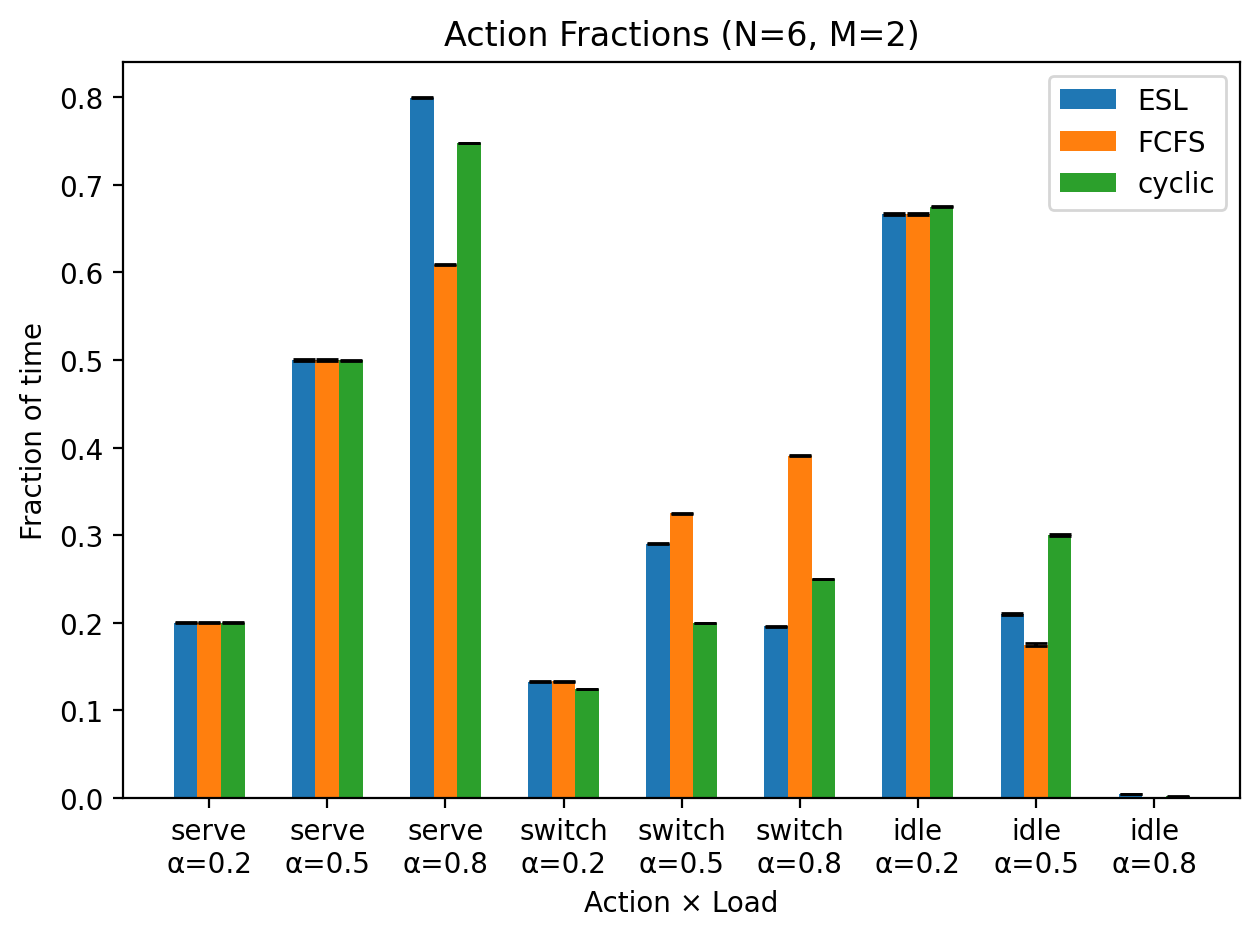}
  \caption{Action fractions, $N=6,\ M=2$}
\end{subfigure}

\vspace{0.6em}

\begin{subfigure}{0.32\textwidth}
  \centering
  \includegraphics[width=\linewidth]{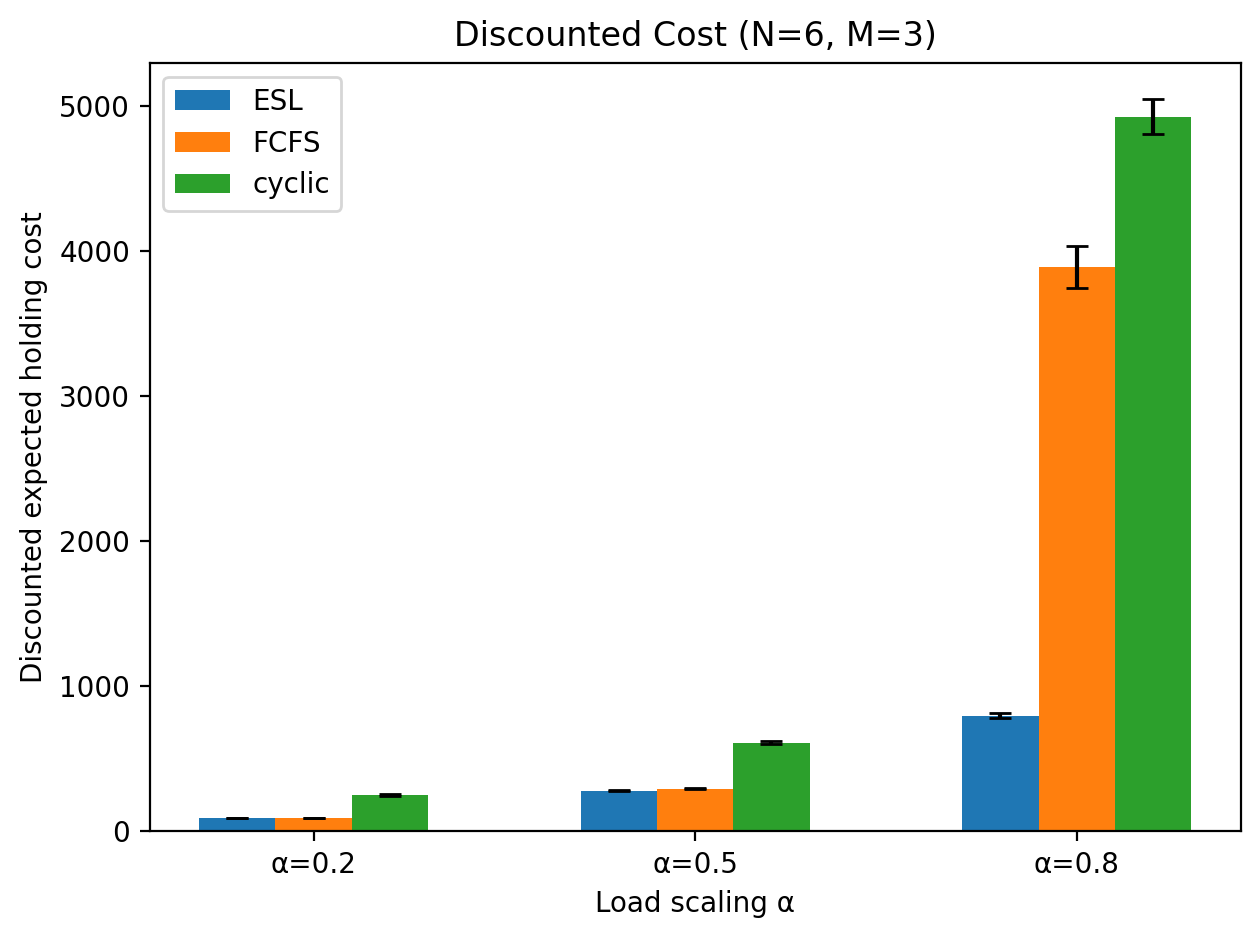}
  \caption{Discounted cost, $N=6,\ M=3$}
\end{subfigure}\hfill
\begin{subfigure}{0.32\textwidth}
  \centering
  \includegraphics[width=\linewidth]{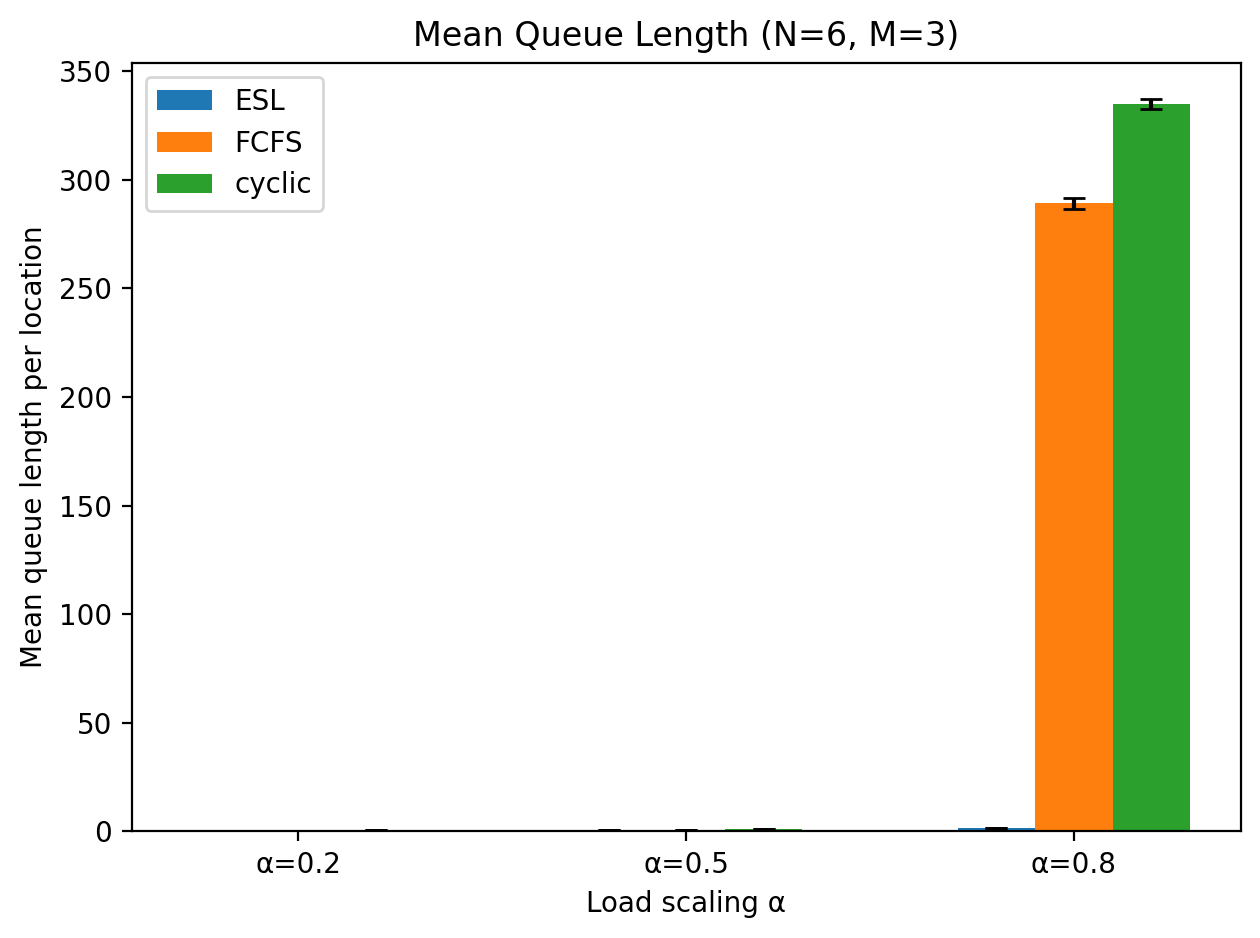}
  \caption{Mean queue length, $N=6,\ M=3$}
\end{subfigure}\hfill
\begin{subfigure}{0.32\textwidth}
  \centering
  \includegraphics[width=\linewidth]{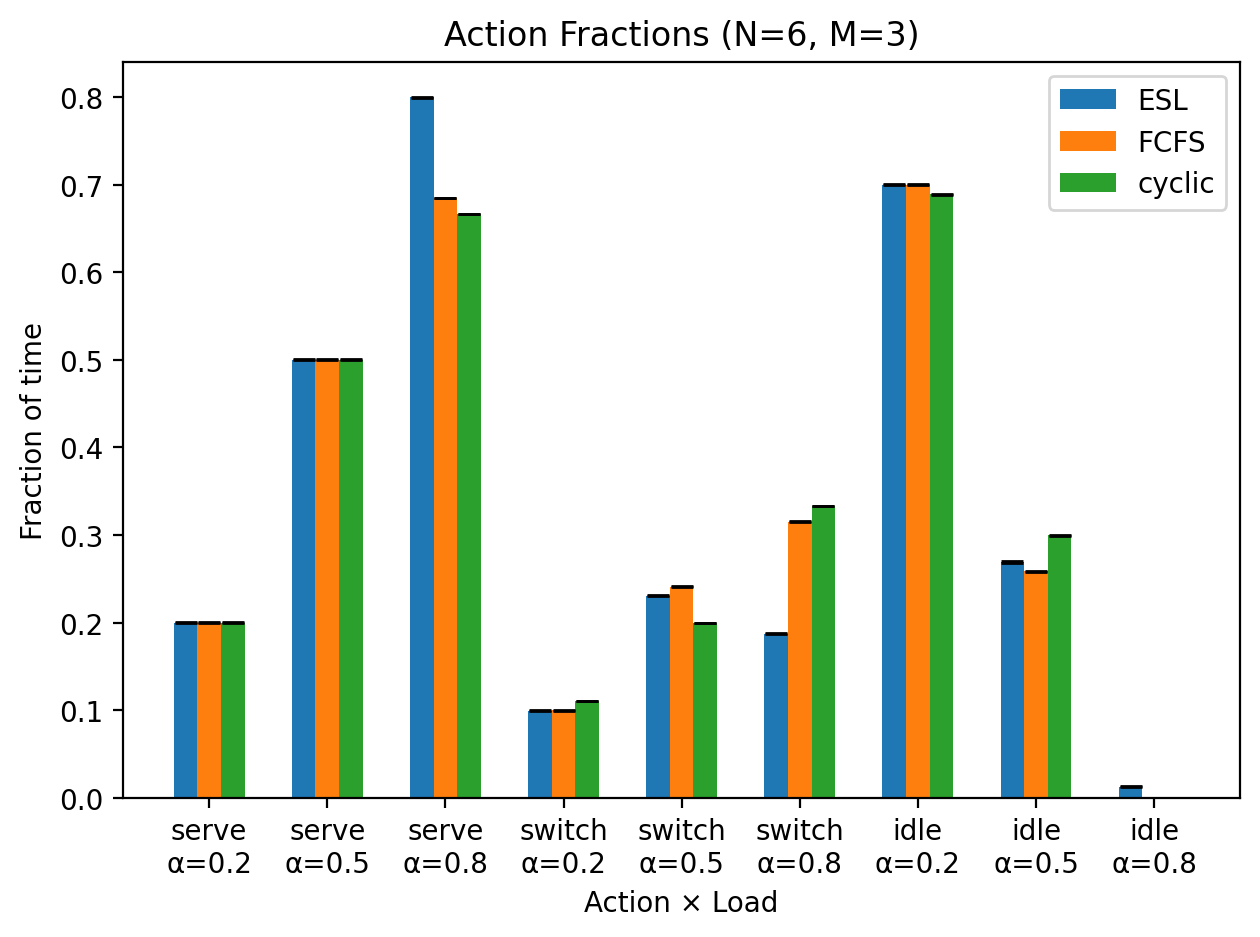}
  \caption{Action fractions, $N=6,\ M=3$}
\end{subfigure}

\caption{Policy comparison across loads $\alpha\in\{0.2,0.5,0.8\}$ for two robot–location ratios. Columns: discounted cost, mean queue length, action fractions. Top row: $M=2$; bottom row: $M=3$.}
\label{fig:summary-grid}
\end{figure*}

\begin{table*}[t]
\centering
\caption{Results for $N=6$, $M=2$ with $p=\alpha\cdot M/N$. Values are mean $\pm$ 95\% CI. Best values are bold.}
\label{tab:m6r2}
\small
\setlength{\tabcolsep}{3pt}
\resizebox{\textwidth}{!}{%
\begin{tabular}{cccccccc}
\hline
$\alpha$ & $p$ & Policy & Discounted Cost & Mean Q Length & Serve & Switch & Idle \\
\hline
0.2 & 0.0667 & ESL    & $\mathbf{70.8507 \pm 1.7885}$ & $\mathbf{0.1191 \pm 0.0004}$ & $\mathbf{0.2002 \pm 0.0006}$ & $0.1329 \pm 0.0005$ & $0.6669 \pm 0.0011$ \\
0.2 & 0.0667 & FCFS   & $71.1520 \pm 1.8040$          & $0.1194 \pm 0.0004$          & $\mathbf{0.2002 \pm 0.0006}$ & $0.1334 \pm 0.0005$ & $0.6664 \pm 0.0011$ \\
0.2 & 0.0667 & cyclic & $293.6874 \pm 8.3682$         & $0.5225 \pm 0.0022$          & $0.2001 \pm 0.0006$          & $\mathbf{0.1250 \pm 0.0000}$ & $0.6749 \pm 0.0006$ \\
\hline
0.5 & 0.1667 & ESL    & $\mathbf{241.9230 \pm 5.4897}$ & $\mathbf{0.4111 \pm 0.0017}$ & $\mathbf{0.4999 \pm 0.0009}$ & $0.2905 \pm 0.0005$ & $0.2096 \pm 0.0012$ \\
0.5 & 0.1667 & FCFS   & $299.4461 \pm 12.5402$         & $0.5327 \pm 0.0054$          & $0.4999 \pm 0.0009$          & $0.3250 \pm 0.0007$ & $0.1751 \pm 0.0015$ \\
0.5 & 0.1667 & cyclic & $677.2539 \pm 14.9217$         & $1.2169 \pm 0.0049$          & $0.4997 \pm 0.0009$          & $\mathbf{0.2000 \pm 0.0000}$ & $0.3003 \pm 0.0009$ \\
\hline
0.8 & 0.2667 & ESL    & $\mathbf{841.5381 \pm 26.4326}$ & $\mathbf{1.6236 \pm 0.0103}$ & $\mathbf{0.7996 \pm 0.0010}$ & $\mathbf{0.1958 \pm 0.0009}$ & $0.0045 \pm 0.0002$ \\
0.8 & 0.2667 & FCFS   & $4141.9576 \pm 131.0263$        & $319.6290 \pm 2.0536$        & $0.6090 \pm 0.0003$          & $0.3908 \pm 0.0003$ & $0.0002 \pm 0.0000$ \\
0.8 & 0.2667 & cyclic & $2629.3946 \pm 84.9285$         & $89.8803 \pm 1.8559$         & $0.7482 \pm 0.0001$          & $0.2500 \pm 0.0000$ & $0.0018 \pm 0.0001$ \\
\hline
\end{tabular}%
}
\end{table*}

\begin{table*}[t]
\centering
\caption{Results for $N=6$, $M=3$ with $p=\alpha\cdot M/N$. Values are mean $\pm$ 95\% CI. Best values are bold.}
\label{tab:m6r3}
\small
\setlength{\tabcolsep}{3pt}
\resizebox{\textwidth}{!}{%
\begin{tabular}{cccccccc}
\hline
$\alpha$ & $p$ & Policy & Discounted Cost & Mean Q Length & Serve & Switch & Idle \\
\hline
0.2 & 0.1000 & ESL    & $\mathbf{92.5771 \pm 1.6886}$ & $\mathbf{0.1567 \pm 0.0004}$ & $\mathbf{0.2002 \pm 0.0005}$ & $\mathbf{0.0999 \pm 0.0004}$ & $0.6999 \pm 0.0008$ \\
0.2 & 0.1000 & FCFS   & $92.6410 \pm 1.6937$          & $0.1568 \pm 0.0004$          & $\mathbf{0.2002 \pm 0.0005}$ & $0.0999 \pm 0.0004$ & $0.6999 \pm 0.0008$ \\
0.2 & 0.1000 & cyclic & $250.8833 \pm 6.5871$         & $0.4393 \pm 0.0016$          & $0.2001 \pm 0.0005$          & $0.1111 \pm 0.0000$ & $0.6888 \pm 0.0005$ \\
\hline
0.5 & 0.2500 & ESL    & $\mathbf{280.1848 \pm 4.2107}$ & $\mathbf{0.4707 \pm 0.0011}$ & $\mathbf{0.5001 \pm 0.0006}$ & $0.2308 \pm 0.0004$ & $0.2691 \pm 0.0009$ \\
0.5 & 0.2500 & FCFS   & $294.7774 \pm 5.6096$          & $0.4951 \pm 0.0015$          & $0.5001 \pm 0.0006$          & $0.2413 \pm 0.0004$ & $0.2586 \pm 0.0009$ \\
0.5 & 0.2500 & cyclic & $610.6562 \pm 10.8861$         & $1.0722 \pm 0.0030$          & $0.5000 \pm 0.0006$          & $\mathbf{0.2000 \pm 0.0000}$ & $0.3000 \pm 0.0006$ \\
\hline
0.8 & 0.4000 & ESL    & $\mathbf{795.8497 \pm 17.6074}$ & $\mathbf{1.4522 \pm 0.0067}$ & $\mathbf{0.7995 \pm 0.0008}$ & $\mathbf{0.1874 \pm 0.0006}$ & $0.0131 \pm 0.0003$ \\
0.8 & 0.4000 & FCFS   & $3890.0326 \pm 145.1912$        & $289.1929 \pm 2.4452$        & $0.6844 \pm 0.0002$          & $0.3154 \pm 0.0002$ & $0.0002 \pm 0.0000$ \\
0.8 & 0.4000 & cyclic & $4924.2697 \pm 118.9448$        & $334.7078 \pm 2.3403$        & $0.6663 \pm 0.0000$          & $0.3333 \pm 0.0000$ & $0.0004 \pm 0.0000$ \\
\hline
\end{tabular}%
}
\end{table*}

\section{CONCLUSIONS}

We presented ESL (Exhaustive-Serve-Longest), a simple, implementable control rule for multi-robot, multi-location service systems with one-slot switching delays, symmetric stochastic arrivals, and service rates. ESL is grounded in two structural results: exhaustive service is strictly better than idling or switching away and, when moving, prefer longer unoccupied locations. Against two strong baselines, per-task FCFS (oldest-first) and an optimized-dwell Cyclic policy, ESL consistently achieved lower discounted holding costs, shorter time-average queues, and more favorable action mixes across server–location ratios and loads; notably, ESL remained robust at high load where FCFS can over-switch and Cyclic under-react. These empirical gains support ESL as a practical default for real-time task allocation. 

The limitations suggest clear next steps: finding a structure for more general cases with (i) asymmetric arrival rates, and (ii) different setup times. Future work will extend analysis and testing to heterogeneous travel/service times, non-identical arrivals and deadlines, and learning-augmented variants toward multi-robot deployments.



\section{APPENDIX}

In this section, we will write down the proofs for the optimality switching to the longest queue and the corollaries for the case of having multiple robots in the system.

\begin{proposition}\label{prop:MsrvNq-longer}
Consider an $M$–server, $N$–queue system with identical Bernoulli(p) arrivals for all locations. Fix a decision epoch at which a particular server $r$ has just exhausted its queue $k'$ (so $x_{k'}=0$). If there exist two other \emph{unoccupied} nonempty queues $i$ and $j$ with $x_i<x_j$, then it is strictly suboptimal to switch $r$ to the shorter queue $i$.
\end{proposition}

\begin{proof}
We argue by contradiction with a sample–path coupling. Without loss of generality relabel $i=1$ and $j=2$. Assume that we have policies $g$, that sends $r$ to queue $1$ and then to $2$ right after exhausting $1$, and policy $\pi$, that sends $r$ to $2$ and $1$, as we will describe later.
Let the initial state be
\[
z(0)=(s_1(0),\dots,s_M(0);\ x_1,x_2,\dots,x_{k'}=0,\dots,x_N),
\]
with $x_2>x_1>0$, and suppose queues $1$ and $2$ are both unoccupied at $t=0$. We define two seperate systems running under policies $g$ and $\pi$.

\smallskip
\noindent\textbf{System $G$ (policy $g$):}
\begin{itemize}
\item At $t=0$, $g$ sends $r$ to queue $1$.
\item For every other server $s\neq r$ and for all $t\ge0$,
\[
u_s^{g}(t)=
\begin{cases}
\textbf{serve}, & x_{s_s(t)}(t)>0,\\[2pt]
\textbf{switch}\bigl(j_s^\star(t)\bigr),  & x_{s_s(t)}(t)=0.
\end{cases}
\]
with $j_s^\star(t) \notin \{i,j\}$, for $t \le \tau+k$ but $j_s^\star(t)$ could be any destination after that time, $\tau+k$ is defined later in the paper.
\item From $t=1$ onward, $r$ serves queue~1 exhaustively. Let $\tau$ be the first time slot at which queue $1$ is empty again under $G$.
\item At time $\tau$, server $r$ switches to queue $2$.
\item From $\tau$ beyond, server $r$ could switch to any server after exhausting.
\end{itemize}

\noindent\textbf{System $\Pi$ (policy $\pi$):}
\begin{itemize}
\item At $t=0$, $\pi$ sends $r$ to queue $2$.
\item For every other server $s\neq r$ and for all $t\ge0$, set $u_s^{\pi}(t)=u_s^{g}(t)$.
\item To align the two trajectories, we couple to an altered arrival process: let $\omega$ be the arrival realization for $G$; assign $\phi_{12}(\omega)$ to $\Pi$, where $\phi_{12}$ swaps the arrival coordinates of queues $1$ and $2$ at all time steps. Since arrivals are i.i.d.\ and symmetric across queues, this swap is measure preserving ($P\circ\phi_{12}^{-1}=P$).
\item From $t=1$ onward, $\pi$ serves queue~2 until time $\tau$. 
\item After $\tau$, $\pi$ serves \emph{exactly} $x_2-x_1$ additional tasks from queue~2 (ignoring any new arrivals after those services), and switches to queue $1$.
\item From this time on, it mimics $g$ with the roles of queues $1$ and $2$ swapped.
\end{itemize}
Because $x^g(0)=x^\pi(0)$ and the swapped coupling gives $A^g(t)=A^\pi(t)$ for all $t$, the value difference is driven purely by cumulative departures:
\begin{align*}
V_g(z_0)-V_\pi(z_0)
&=\mathbb E \left[\sum_{t=0}^{\infty}\beta^t\bigl(x^g(t)-x^\pi(t)\bigr)\right] \\
&=\mathbb E \left[\sum_{t=0}^{\infty}\beta^t\bigl(D^\pi(t)-D^g(t)\bigr)\right]  
\end{align*}

Up to time $\tau$ both systems get the same cost, and hence
\begin{align*}
&V_g(z(0)) - V_\pi(z(0)) \\
&= \beta^\tau \bigl[V_g(1,\dots,s_M(\tau);0,x_2+A_2(\tau),\dots) \\
&- V_\pi(2,\dots,s_M(\tau);x_1+A_2(\tau),\,x_2+A_1(\tau)-\tau,\dots)
    \bigr]
\end{align*}
Based on definition $x_1{+}A_1(\tau)-\tau=0$, and we can write $A_1(\tau)-\tau=-x_1$ and using $x_2-x_1=k$ we obtain
\begin{align*}
&V_g(z(0)) - V_\pi(z(0)) \\
&= \beta^\tau \bigl[V_g(1,\dots,s_M(\tau);0,x_2+A_2(\tau),\dots) \\
&- V_\pi(2,\dots,s_M(\tau);x_2-k+A_2(\tau),k,\dots)\bigr]
\end{align*}
Because $g$ switches away $k$ slots earlier than $\pi$, and $\pi$ does $k$ services before switching to queue $1$, we have a strict interval of slots on which $D^\pi(t)-D^g(t)\ge 1$. Hence
\begin{align*}
x^g(t)-x^{\pi}(t) &=\bigl[x^g(0)-x^\pi(0)\bigr] - \bigl[D^g(t)-D^\pi(t)\bigr] \notag \\
& = D^\pi(t)-D^g(t) = \begin{cases}
  0  & t\le \tau \\
  1   & \tau<t\le\tau+k \\
  0   & t > \tau+k
\end{cases}
\end{align*}
If we write down the difference between the discounted costs at time $\tau+k+1$, we'll get
\begin{align*}
&V_g(z(0)) - V_\pi(z(0)) \\
& = \sum_{t=\tau+1}^{\tau+k} \beta^t + \beta^{\tau+k+1} \bigl[V_g(1,\dots,s_M(\tau);A_1(\tau+k+1)-\\
& A_1(\tau), x_2+A_2(\tau+k+1)-k),\dots) - V_\pi(1,\dots,s_M(\tau);\\
& x_2+A_2(\tau+k+1)-k,A_1(\tau+k+1)-A_1(\tau),\dots)\bigr]
\end{align*}
This shows that the systems are coupled for $t\ge\tau+k+1$, except with the queue $1$ and $2$ swapped, and therefore:
\begin{align*}
&V_g(z(0)) - V_\pi(z(0)) = \sum_{t=\tau+1}^{\tau+k} \beta^t 
\end{align*}
Since this is true for every possible sample path:
\begin{align*}
&V_g(z(0)) - V_\pi(z(0)) = \mathbb{E}\Bigl[\sum_{t=\tau+1}^{\tau+k} \beta^t \Bigr] > 0
\end{align*}
Which means that policy $g$ is sub-optimal.
\end{proof}

\bibliographystyle{IEEEtran}
\bibliography{sources}

\end{document}